\documentclass{article} 
\usepackage{times}  
\usepackage{helvet}  
\usepackage{courier}  
\usepackage[hyphens]{url}
\usepackage{graphicx} 
\usepackage{subfigure}
\usepackage{natbib}
\usepackage{caption}  

\usepackage[utf8]{inputenc}
\usepackage{tikz}
\usepackage{amsmath, amsthm, mathrsfs, mathtools, amsfonts}
\usepackage{nicefrac}
\usepackage{MnSymbol}
\usepackage{algorithm}
\usepackage{algorithmicx}
\usepackage[noend]{algpseudocode}
    \algrenewcommand\algorithmicindent{1em}
\usepackage{xcolor}
\usepackage{aligned-overset} 
\usepackage{authblk}
\usepackage{fullpage}
\usepackage[capitalize,noabbrev]{cleveref} 

\newcommand{\cI}{\mathcal{I}}
\newcommand{\cJ}{\mathcal{J}}

\newcommand{\cP}{\mathcal{P}}

\newcommand{\I}{\mathbb{I}}

\newcommand{\N}{\mathbb{N}}

\newcommand{\bbR}{\mathbb{R}}

\newcommand{\tht}{\vartheta}
\newcommand{\lrb}[1]{\left(#1\right)}
\newcommand{\brb}[1]{\bigl(#1\bigr)}

\newcommand{\bsb}[1]{\bigl[#1\bigr]}

\newcommand{\lcb}[1]{\left\{#1\right\}}
\newcommand{\bcb}[1]{\bigl\{#1\bigr\}}

\newcommand{\lfl}[1]{\left\lfloor#1\right\rfloor}
\newcommand{\bfl}[1]{\bigl\lfloor#1\bigr\rfloor}

\newcommand{\labs}[1]{\left\lvert#1\right\rvert}
\newcommand{\babs}[1]{\bigl\lvert#1\bigr\rvert}

\DeclareMathOperator*{\argmin}{argmin}

\newcommand{\lip}{Lipschitz}

\newcommand{\s}{\subset}

\newcommand{\iop}{\infty}
\newcommand{\nhphantom}[1]{\sbox0{#1}\hspace{-\the\wd0}}
\newcommand{\mypapertitle}{Regret Analysis of \algName{}\\ (Preliminary work)}
\newcommand{\fB}{\mathfrak{B}}

\newcommand{\cancL}{\text{$\filledmedsquare${\hspace{-1pt}}$\medsquare${\hspace{-1pt}}$\square${\hspace{-1pt}}$\medsquare$}}
\newcommand{\cancR}{\text{$\medsquare${\hspace{-1pt}}$\medsquare${\hspace{-1pt}}$\medsquare${\hspace{-1pt}}$\filledmedsquare$}}
\newcommand{\cancLL}{\text{$\filledmedsquare${\hspace{-1pt}}$\filledmedsquare${\hspace{-1pt}}$\medsquare${\hspace{-1pt}}$\medsquare$}}
\newcommand{\cancRR}{\text{$\medsquare${\hspace{-1pt}}$\medsquare${\hspace{-1pt}}$\filledmedsquare${\hspace{-1pt}}$\filledmedsquare$}}
\newcommand{\cancLR}{\text{$\filledmedsquare${\hspace{-1pt}}$\medsquare${\hspace{-1pt}}$\medsquare${\hspace{-1pt}}$\filledmedsquare$}}
\newcommand{\cancNone}{\text{$\medsquare${\hspace{-1pt}}$\medsquare${\hspace{-1pt}}$\medsquare${\hspace{-1pt}}$\medsquare$}}
\newcommand{\algName}{Dyadic Search}
\newcommand{\fupd}{\mathrm{update}}
\newcommand{\fdel}{\mathrm{delete}}

\newcommand{\unif}{\textcolor{black!50!gray}{\mathfrak{u}}}
\newcommand{\nunif}{\textcolor{black!50!gray}{\not{\mathfrak{u}}}}
\newcommand{\del}{\mathrm{del}}
\newcommand{\ts}{\tau^\#}

\newtheorem{theorem}{Theorem}
\newtheorem{claim}{Claim}
\newtheorem{assumption}{Assumption}%

\title{\mypapertitle{}\thanks{This is a preliminary (and unpolished) version of our regret analysis of \algName{}. Stay tuned for the final polished paper.}}

\author[1]{Fran\c{c}ois Bachoc}
\author[2,3]{Tommaso Cesari}
\author[2,4]{Roberto Colomboni}
\author[2,4]{Andrea Paudice}

\affil[1]{Institut de Math\'ematiques de Toulouse, Toulouse, France}
\affil[2]{Universit\`a degli Studi di Milano, Milano, Italy}
\affil[3]{Toulouse School of Economics, Toulouse, France}
\affil[4]{Istituto Italiano di Tecnologia, Genova, Italy}

\begin{document}

\maketitle

\begin{abstract}
We analyze the cumulative regret of the \algName{} algorithm of \citet{bachoc2022aNearOptimal}.
\end{abstract}

\section{Setting}\label{s:setting}

In this section, we introduce the formal setting for our budget convex optimization problem.

Given a bounded interval $I\s \bbR$, our goal is to minimize an unknown \emph{convex} function $f\colon I \to \bbR$ picked by a possibly adversarial and adaptive environment by only requesting fuzzy evaluations of $f$.
At every interaction $t$, the optimizer is given a certain budget $b_t$ that can be invested in a query point $X_t$ of their choosing to reduce the fuzziness of the value of $f(X_t)$, modeled by an interval $J_t \ni f(X_t)$.

The interactions between the optimizer and the environment are described in Optimization~Protocol~\ref{a:protocol}.

{
\makeatletter
\renewcommand{\ALG@name}{Optimization Protocol}
\makeatother

\begin{algorithm}
\caption{\label{a:protocol}}
\textbf{input:} A non-empty bounded interval $I\s\bbR$ (the domain of the unknown objective $f$)
\begin{algorithmic}[1]
\For{$t=1,2,\dots$}
    \State The environment picks and reveals a budget $b_t > 0$
    \State The optimizer selects a query point $X_t \in I$ where to invest the budget $b_t$
    \State The environment picks and reveals an interval $J_t \s \bbR$ such that $f(X_t) \in J_t$
\EndFor 
\end{algorithmic}
\end{algorithm}
}

We stress that the environment is adaptive. 
Indeed, the intervals $J_t$ that are given as answers to the queries $X_t$ can be chosen by the environment as an arbitrary function of the past history, as long as they represent fuzzy evaluations of the convex function $f$, i.e., $f(X_t) \in J_t$.

Note that optimization would be impossible without further restrictions on the behavior of the environment, since an adversarial environment could return $J_t=\mathbb R$ for all $t\in \N$, making it impossible to gather any meaningful information.
We limit the power of the environment by relating the amount of budget invested in a query point $X_t$ with the length of the corresponding fuzzy representation $J_t$ of $f(X_t)$.
The idea is that the more budget is invested, the more accurate approximation of the objective $f$ can be determined, in a quantifiable way.
This is made formal by the following assumption.
\begin{assumption}
\label{ass:budget}
There exist $c \ge 0$ and $\alpha > 0$ such that, for any $t \in \N$, if the optimizer invested the budgets $b_1, \dots, b_t$ in the query points $X_1, \dots, X_t$, then
\[
    \labs{ J_t }
\le
    \frac{c}{\fB_t^\alpha} \;,
\]
where $\labs{J_t}$ denotes the length of $J_t$ and $\fB_t \coloneqq \sum_{s=1}^t b_s\I\{X_s=X_t\}$ is the total budget invested in $X_t$ up to time $t$.
\end{assumption}

The performance after $T$ interactions of an algorithm that received budgets $b_1,\dots,b_T$ is evaluated with the cumulative regret.
More precisely, we want to control the difference  
\[
    R_T
\coloneqq    
    \sum_{t=1}^T f(X_t) b_t - \inf_{x\in I} \sum_{t=1}^T f(x) b_t
\] 
for any choice of the convex function $f$ and the fuzzy evaluations $J_1,\dots,J_T$.

\section{\algName}\label{s:algo}
In this section, we present our \algName{} algorithm for budget convex optimization (\Cref{a:dyadic}).

Before presenting its pseudo-code, we introduce some notation.
For any positive integer $n\in \N$ we denote by $[n]$ the set $\{1,\dots, n\}$ of the first $n$ integers.
Let $\cP \coloneqq \{ \cancL, \cancR, \cancLL, \cancRR, \cancLR \}$.
The blackened parts of the elements of $\cP$ represent which portions of the active interval maintained by \algName{} the algorithm will delete.
Additionally, we will consider the element $\cancNone$ representing the case where no parts of the active interval will be deleted.
Let $\cJ$ be the set of all intervals, and $\cI \s \cJ$ that of all \emph{bounded} intervals.
Furthermore, for any interval $J\in \cJ$, let
\[
    J^- \coloneqq \inf (J)
    \qquad \text{and} \qquad
    J^+ \coloneqq \sup (J) \;.
\]
\algName{} relies on four auxiliary functions: the $\fdel$ function, the uniform partition function $\unif$, the non-uniform partition function $\nunif$, and the $\fupd$ function.
The $\fdel$ function
\[
    \fdel \, \colon \, \cJ^3 \to \cP \cup \{ \cancNone \}
\]
is defined, for all $(J_l, J_{c}, J_{r}) \in \cJ^3$, by
\[
    \begin{cases}
    \cancLL & \text{if } J_c^- \ge J_r^+\text{, else}\\
    \cancRR & \text{if } J_c^- \ge J_l^+\text{, else}\\
    \cancLR & \text{if } J_l^- \ge \min ( J_c^+, J_r^+ ) \text{ and } J_r^- \ge \min ( J_l^+, J_c^+)\text{, else}\\
    \cancL & \text{if } J_l^- \ge \min ( J_c^+, J_r^+ )\text{, else}\\
    \cancR & \text{if } J_r^- \ge \min ( J_l^+, J_c^+ )\text{, else}\\
    \cancNone  & .
    \end{cases}
\]
In words, the intervals $J_l, J_c, J_r$ will represent the fuzzy evaluations of three points $l < c < r$ in the domain of the unknown objective (left, center, and right).
Since we are assuming that the objective is convex (hence unimodal\footnote{By unimodal, we mean that there exists a point $x$ belonging to the closure of the domain of $f$ such that $f$ is nonincreasing before $x$ and nondecreasing after $x$. More precisely, either $f$ is nonincreasing on the domain intersected with $(-\iop,x]$ and nondecreasing on the domain intersected with $(x,\iop)$ or is nonincreasing on the domain intersected with $(-\iop,x)$ and nondecreasing on the domain intersected with $[x,\iop)$.}), note that whenever an upper bound on the value of the objective at a point $x$ is lower than the lower bound at another point $y$ that is left (resp., right) of $x$, then, all points that are left (resp., right) of $y$ ($y$ included) are no better than $x$.
Therefore, the function $\fdel$ returns which part of an interval containing three distinct points $l < c < r$ should be deleted given the fuzzy evaluations $J_l, J_c, J_r$. (E.g., $\cancLL$ represents the deletion of all points of the active interval left of $c$,
$\cancR$ represents the deletion of all points of the active interval right of $r$,
$\cancNone$ is returned when the fuzzy evaluations are not sufficient to delete anything, etc.)

The uniform and non-uniform partition functions are defined, respectively, by
\begin{align*}
    \unif \, \colon \, \cI & \textstyle{ \to \bbR^3 \;, \quad I \mapsto \brb{ \frac34 I^- + \frac14 I^+, \, \frac12 I^- + \frac12 I^+, \, \frac14 I^- + \frac34 I^+ } } \;, \\
    \nunif \, \colon \, \cI & \textstyle{ \to \bbR^3 \;, \quad I \mapsto \brb{ \frac23 I^- + \frac13 I^+, \, \frac12 I^- + \frac12 I^+, \, \frac13 I^- + \frac23 I^+ } } \;.
\end{align*}
In words, when applied to an interval $I$, the uniform partition function $\unif$ returns the three points that are at $\nicefrac14$,  $\nicefrac12$, and $\nicefrac34$ of the interval, while the non-uniform partition function $\nunif$ returns the three points that are at $\nicefrac13$, $\nicefrac12$, and $\nicefrac23$ of the interval (see \Cref{f:unif-nunif}).
\begin{figure}
    \centering
    \begin{tikzpicture}[scale=3]
    \draw (0,0) -- (1,0);
    \foreach \x in {0,1}
    {
        \draw (\x, -0.05) -- (\x, 0.05);
    }
    \foreach \x in {1/4,1/2,3/4}
    {
        \draw[blue] (\x, -0.05) -- (\x, 0.05);
    }
    \draw 
        (0,-0.05) node[below] {$0$}
        (1/4,-0.05) node[below, blue] {$\nicefrac14$}
        (1/2,-0.05) node[below, blue] {$\nicefrac12$}
        (3/4,-0.05) node[below, blue] {$\nicefrac34$}
        (1,-0.05) node[below] {$1$}
    ;
    \draw (1/2,0.3) node {$\unif$};
    \end{tikzpicture}
    \qquad
    \begin{tikzpicture}[scale=3]
    \draw (0,0) -- (1,0);
    \foreach \x in {0,1}
    {
        \draw (\x, -0.05) -- (\x, 0.05);
    }
    \foreach \x in {1/3,1/2,2/3}
    {
        \draw[blue] (\x, -0.05) -- (\x, 0.05);
    }
    \draw 
        (0,-0.05) node[below] {$0$}
        (1/3,-0.05) node[below, blue] {$\nicefrac13$}
        (1/2,-0.05) node[below, blue] {$\nicefrac12$}
        (2/3,-0.05) node[below, blue] {$\nicefrac23$}
        (1,-0.05) node[below] {$1$}
    ;
    \draw (1/2,0.3) node {$\nunif$};
    \end{tikzpicture}
    \caption{The uniform ($\protect\unif$) and non-uniform ($\protect\nunif$) partition functions applied to the interval $I=[0,1]$.}
    \label{f:unif-nunif}
\end{figure}
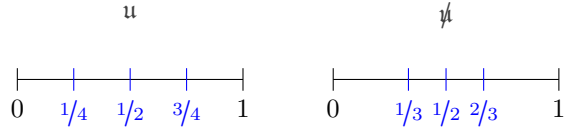

The $\fupd$ function
\[
    \fupd \, \colon \, \cI \times \{ \unif, \nunif \} \times \cP \to \cI \times \{ \unif, \nunif \}
\]
is defined, for all $(I,\tht,\del) \in  \cI \times \{ \unif, \nunif \} \times \cP$, by the following table:
\[
\begin{matrix}
& \unif & \nunif \\
\cancLL & \brb{ \bsb{ \frac12I^-+\frac12I^+, \, I^+ } , \, \unif  } & \brb{ \bsb{ \frac12I^-+\frac12I^+, \, I^+ } , \, \nunif  } \\
\cancRR & \brb{ \bsb{ I^-, \, \frac12I^-+\frac12I^+ } , \, \unif  } & \brb{ \bsb{ I^-, \, \frac12I^-+\frac12I^+ } , \, \nunif  } \\
\cancLR & \brb{ \bsb{ \frac{3I^- + I^+}{4}, \, \frac{I^- + 3I^+}{4} }, \, \unif  } & \brb{ \bsb{ \frac{2I^- + I^+}{3}, \, \frac{I^- + 2I^+ }{3} } , \, \unif  }\\
\cancL & \brb{ \bsb{ \frac34I^-+\frac14I^+, \, I^+ } , \, \nunif  } & \brb{ \bsb{ \frac23I^-+\frac13I^+, \, I^+ } , \, \unif  } \\
\cancR & \brb{ \bsb{ I^-, \, \frac14I^-+\frac34I^+ } , \, \nunif  } & \brb{ \bsb{ I^-, \, \frac13I^-+\frac23I^+ } , \, \unif  } \\
\end{matrix}
\]
In words, when applied to an interval $I$, a type of partition $\tht$, and the subset of $I$ to be deleted modeled by $\del$, the $\fupd$ function returns as the first component the interval $I$ pruned of the subset of $I$ specified by $\tht$ and $\del$, and, as the second component, how the new interval will be partitioned.
It can be seen that the types of partitions returned by $\fupd$ are chosen so that our \algName{} algorithms will only query points on a (rescaled) dyadic mesh. (E.g., if $I=[0,1]$, \algName{} will only query points of the form $k/2^{h}$, for $k,h\in \N$.)

For all $t \in \N$, if the sequence of budgets picked by the environment up to time $t$ is $b_1,\dots,b_t$ and the sequence of query points selected by the optimizer is $X_1,\dots,X_t$, for each $x\in \bbR$, we define the quantities
\[
    \fB_{x,t}
\coloneqq
    \sum_{s=1}^t b_s \I \{X_s = x\}
\qquad
\text{and}
\qquad
    J_{x,t}
\coloneqq
    \bigcap_{s\in[t], X_s = x} J_s
\]
with the understanding that $J_{x,t}=\bbR$ whenever $X_s \neq x$ for all $s \in [t]$.
Furthermore, define $\fB_{x,0} = 0$ for all $x\in \bbR$.
In words, $\fB_{x,t}$ is the total budget that has been invested in $x$ by the optimizer up to and including time $t$, while $J_{x,t}$ is the best fuzzy evaluation of the unknown objective at $x$ that is available at the end of time $t$.

The pseudocode of \algName{} is provided in \Cref{a:dyadic}.

\begin{algorithm}
\caption{\label{a:dyadic}\algName}
\textbf{input:} A non-empty bounded interval $I\s\bbR$ (the domain of the unknown objective)

\textbf{initialization:} $I_1 \coloneqq [I^-,I^+]$, $\tht_1 \coloneqq \unif$, $(l_1,c_1,r_1) \coloneqq \tht_1(I_1)$, $t_0 \coloneqq 0$ [and $B_0 \coloneqq 0$, $B_{1,0} \coloneqq 0$]
\begin{algorithmic}[1]
\For{epochs $\tau=1,2,\dots$}
    \For{$t=t_{\tau-1}+1,t_{\tau-1}+2,\dots$}
        \State Query $X_t \in \argmin_{x \in \{ l_\tau, c_\tau, r_\tau \} } \fB_{x,t-1}$ \label{s:query}
        \State Let $\del_t \coloneqq \fdel( J_{l_\tau,t}, J_{c_\tau,t}, J_{r_\tau,t} )$
        \State [Let $B_{\tau,t} \coloneqq B_{\tau,t-1}+b_t$ and $\tau_t \coloneqq \tau$] \label{state:extra-command}
        \If{$\del_t \neq \cancNone$\label{state:recommendation-one}}
            \State [Let $t_\tau \coloneqq t$, $B_\tau \coloneqq B_{\tau,t}$, and $B_{\tau+1,t} \coloneqq 0$] \label{state:extra-command-two}
            \State Let $(I_{\tau+1},\tht_{\tau+1}) \coloneqq \fupd (I_\tau,\tht_\tau,\del_t)$ \label{state:update}
            \State Let $(l_{\tau+1},c_{\tau+1},r_{\tau+1}) \coloneqq \tht_{\tau+1}(I_{\tau+1})$ \label{s:query-points}
            \State \textbf{break}
        \EndIf
    \EndFor
    \label{s:repeat-end}
\EndFor 
\end{algorithmic}
\end{algorithm}

We note that the assignments in brackets in the initialization and \Cref{state:extra-command,state:extra-command-two} are not needed to run the algorithm. 
We only added them for notational convenience of the analysis.

As noted above, by definition of the $\fupd$ function, \algName{} only queries points in the rescaled dyadic mesh $\bcb{ I^- + k \cdot 2^{-h} \cdot \labs I : h \in \N, k \in [2^h-1] }$.
Moreover, we stress that \algName{} is any-time (it does not need to know the time horizon $T$ \emph{a priori}), any-budget (it does not need to know the total budget $B \coloneqq \sum_{t=1}^T b_t$) and does not require the unknown objective to be 
\lip{}.

\section{Cumulative Regret Analysis}
\label{s:upper}

\begin{theorem}
\label{t:upper-bound-regret}
For any compact interval $I\s\bbR$, if the optimizer is running \algName{} (\Cref{a:dyadic}) with input $I$ in an environment satisfying Assumption \ref{ass:budget} for some $c\ge 0$ and $\alpha>0$, then, there exist $c_1, c_2 >0$ such that, for any time $T\in \N$ and every convex continuous function $f\colon I \to \bbR$, if budgets $b_t$ are equal to $1$ for all $t\in \N$, the regret $R_T$ satisfies
\begin{equation}
\label{e:upper-bound}
    R_T
\le
    c_1 \cdot T^{1-\alpha} \brb{ c \ln (MT) + 1}
    +
    c_2 \cdot M \;,
\end{equation}
where $M \coloneqq \max(f) - \min(f)$.
\end{theorem}

\begin{proof}
Fix a compact interval $I$, a time horizon $T$, and a convex continuous function $f\colon I \to \bbR$.
Up to translating and rescaling, we can (and do!) assume without loss of generality that $\min(f)=0$ and $I = [0,1]$.
We also assume that $f$ admits a unique minimizer $x^\star \in (0,1)$ (the other cases are simpler).
Redefine $t_{\tau_T} \coloneqq T$ and $B_{\tau_T} \coloneqq B_{\tau_T,T}$.
\begin{claim}
\label{c:claim-1}
For $\tau \in [\tau_T]$, if $B_{\tau} \ge 4$ (i.e., if epoch $\tau$ lasts at least $4$ rounds), then
\[
    \max_{x\in \{l_{\tau}, c_{\tau}, r_{\tau} \}} f(x)
\le
    \frac{4 c 3^\alpha}{ (B_{\tau} -3)^\alpha }
\]
\end{claim}
\begin{proof}[Proof of \Cref{c:claim-1}]
Fix any epoch $\tau\in [\tau_T]$ and assume that $B_{\tau}\ge 4$.
Remember that, by \citet[Eq.\ (9)]{bachoc2022aNearOptimal}, we have
\[
    \min_{x\in\{l_{\tau},c_{\tau},r_{\tau}\}} \sum_{s=1}^{t_\tau-1} \I\{X_s = x\}
\ge
    \frac{B_\tau - 3}3 \;.
\]
Assume that $x^\star > r_\tau$ (all other cases can be treated similarly), which in turn implies that $\max_{x\in\{l_{\tau},c_{\tau},r_{\tau}\}} f(x) = f(l_\tau)$.
Then, recalling that at time $t_\tau - 1$, it holds that $J_{l_\tau, t_\tau-1} \cap J_{c_\tau, t_\tau-1} \cap J_{r_\tau, t_\tau-1} \neq \varnothing$ (implying in particular that $J^+_{r_\tau, t_\tau-1} - J^-_{l_\tau, t_\tau-1} \ge 0$), we get
\begin{align*}
    \max_{x\in\{l_{\tau},c_{\tau},r_{\tau}\}} f(x)
&
=
    f(l_\tau)
=
    f(l_\tau) - f(x^\star)
=
    f(l_\tau) - f(r_\tau) + \frac{f(r_\tau) - f(x^\star)}{r_\tau - x^\star}(r_\tau - x^\star)
\\
&
\le
    f(l_\tau) - f(r_\tau) + \frac{f(l_\tau) - f(r_\tau)}{l_\tau - r_\tau}(r_\tau - x^\star)
=
    \frac{x^\star - l_\tau}{r_\tau-l_\tau}\brb{ f(l_\tau) - f(r_\tau) }
\le
    2 \brb{ f(l_\tau) - f(r_\tau) }
\\
&
\le
    2 \brb{ J^+_{l_\tau, t_\tau-1} - J^-_{r_\tau, t_\tau-1} }
\le
    2 \brb{ J^+_{l_\tau, t_\tau-1} - J^-_{l_\tau, t_\tau-1} + J^+_{r_\tau, t_\tau-1} - J^-_{r_\tau, t_\tau-1} }
\\
&
\le
    4 \max \bcb{ \labs{ J_{l_\tau, t_\tau-1}  }, \, \labs{ J_{c_\tau, t_\tau-1} }, \, \labs{ J_{r_\tau, t_\tau-1} } }
\\
&
\le
    4 \max \lcb{ \frac{c}{\lrb{ \sum_{s=1}^{t_\tau -1} \I \{X_s = l_\tau\} }^\alpha }, \, \frac{c}{\lrb{ \sum_{s=1}^{t_\tau -1} \I \{X_s = c_\tau\} }^\alpha}, \, \frac{c}{\lrb{ \sum_{s=1}^{t_\tau -1} \I \{X_s = r_\tau\} }^\alpha} }
\\
&
=
     \frac{4c}{ \lrb{ \min_{x\in\{l_{\tau},c_{\tau},r_{\tau}\}} \sum_{s=1}^{t_\tau-1} \I\{X_t = x\} }^\alpha }
\le
    \frac{4c}{\lrb{\frac{B_\tau -3}3}^\alpha}
=
    \frac{4 c 3^\alpha}{ (B_{\tau} -3)^\alpha } \;.
\end{align*}

\end{proof}
Let $\tau^\star \in [\tau_T]$ be the first epoch from which $x^\star \in [l_\tau, r_\tau]$.

\begin{claim}
\label{c:claim-2}
If $\tau^* \ge 2$, then, for each $\tau \in \{2,\dots, \tau^\star-1\}$,
\[
    \max_{x\in \{l_{\tau}, c_{\tau}, r_{\tau} \}} f(x)
\le
    \frac34 \lrb{ \max_{x\in \{l_{\tau-1}, c_{\tau-1}, r_{\tau-1} \}} f(x) }
\]
\end{claim}
\begin{proof}[Proof of \Cref{c:claim-2}]
Assume that $\tau^* \ge 2$. Then, either for all $\tau \in [\tau^\star-1]$, it holds that $r_\tau < x^\star$, or for all $\tau \in [\tau^\star-1]$, we have $l_\tau > x^\star$. 
In the first case, for all $\tau \in \{2,\dots, \tau^\star-1\}$,
\begin{align*}
    \max_{x\in \{l_{\tau}, c_{\tau}, r_{\tau} \}} f(x)
&
=
    f(l_\tau) - f(x^\star)
=
    \frac{ f(l_\tau) - f(x^\star) }{ l_\tau - x^\star } (l_\tau - x^\star)
\le
    \frac{ f(l_{\tau-1}) - f(x^\star) }{ l_{\tau-1} - x^\star } (l_\tau - x^\star)
\\
&
=
    \frac34 \brb{ f(l_{\tau-1}) - f(x^\star) }
=
    \frac34 f(l_{\tau-1})
=
    \frac34 \max_{x\in \{l_{\tau-1}, c_{\tau-1}, r_{\tau-1} \}} f(x) \;.
\end{align*}
The other case can be worked out similarly.
\end{proof}
For each $m \in \N$, let $A_m \coloneqq \bcb{ x\in (0,1) : \exists k \in [2^m-1], x = k/2^m }$ be the dyadic mesh in $(0,1)$ of index $m$.
For any epoch $\tau \in \N$, let $m_\tau \coloneqq - \log_2 ( c_\tau - l_\tau )$ be the index of the dyadic mesh in $(0,1)$ at epoch $\tau$ of \algName{} (note that $m_\tau \ge 2$ for all $\tau \in \N$ because \algName{} begins with a step-size of $1/4$).

Note that:
\begin{itemize}
    \item If the epoch $\tau^\star$ is non-uniform, then, then previous epoch has to be non-uniform as well and as soon as we change the dyadic mesh (in at most two epochs) we have 4 dyadic points in $(0,1)$ to both sides of $x^\star$.
    \item If the epoch $\tau^\star$ is uniform, then, then previous epoch can be either uniform or non-uniform.
        \begin{itemize}
            \item If the previous epoch is non-uniform, then as soon as we change the dyadic mesh twice (in at most three epochs) we have 4 dyadic points in $(0,1)$ to both sides of $x^\star$.
            \item If the previous epoch is uniform, then as soon as we change the dyadic mesh twice (in at most three epochs) we have 4 dyadic points in $(0,1)$ to both sides of $x^\star$.
        \end{itemize}
\end{itemize}
Let $m^\star \coloneqq \min \bcb{m \in \N : \babs{ A_m \cap (0,x^\star] } \ge 4 \text{ and } \babs{ A_m \cap [x^\star, 1) } \ge 4 }$ be the smallest index of the dyadic mesh in $(0,1)$ such that there are at least 4 points of the dyadic mesh in $(0,1)$ to the right and to the left of $x^\star$.
For each $m \ge m^\star$ let $x_1^m < x_2^m < x_3^m < x_4^m \le x^\star$ be the four points of $A_m \cap (0,x^\star]$ closest to $x^\star$ and $x^\star \le x_5^m < x_6^m < x_7^m < x_8^m$ be the four points of $A_m \cap [x^\star, 1)$ closest to $x^\star$.
The crucial observation is that, for all epochs $\tau \ge \tau^\star + 3$, we have that $l_\tau,c_\tau,r_\tau \in \{x_1^{m_\tau}, \dots, x_8^{m_\tau}\}$.

\begin{claim}
\label{c:claim-3}
For each $m \ge m^\star+1$, we have
\[
    \max_{x \in \{x_1^{m}, \dots, x_8^{m}\}} f(x)
\le
    \frac47 \lrb{ \max_{x \in \{x_1^{m-1}, \dots, x_8^{m-1}\}} f(x) } \;.
\]
\end{claim}
\begin{proof}[Proof of \Cref{c:claim-3}]
Assume that $m \ge m^* +1$. 
Then, either $\max_{ x \in \{x_1^{m}, \dots, x_8^{m}\} } f(x) = f(x_1^m)$ or $\max_{ x \in \{x_1^{m}, \dots, x_8^{m}\} } f(x) = f(x_8^m)$. 
In the first case, we have
\begin{align*}
    \max_{ x \in \{x_1^{m}, \dots, x_8^{m}\} } f(x)
&
=
    f(x_1^{m}) - f(x^\star)
=
    \frac{ f(x_1^{m}) - f(x^\star) }{ x_1^{m} - x^\star } (x_1^{m} - x^\star)
\le
    \frac{ f(x_1^{m-1}) - f(x^\star) }{ x_1^{m-1} - x^\star } (x_1^{m} - x^\star)
\\
&
=
    \frac47 \brb{ f(x_1^{m-1}) - f(x^\star) }
=
    \frac47 f(x_1^{m-1})
\le
    \frac47 \max_{x \in \{x_1^{m-1}, \dots, x_8^{m-1}\}} f(x) \;.
\end{align*}
The other case can be worked out similarly.
\end{proof}
Define $\ts \coloneqq \bfl{ 4 + 2 \log_{\nicefrac43}(MT^\alpha) }$ so that
\[
    M \lrb{\frac34}^{ \lfl{ \frac{\ts -1}{2} } }
=
    M \lrb{\frac34}^{ \lfl{ \frac{\bfl{ 4 + 2 \log_{\nicefrac43}(MT^\alpha) } -1}{2} } }
\le
    M \lrb{\frac34} ^ { \log_{\nicefrac43}(MT^\alpha) }
=
    M \frac{1}{MT^\alpha}
=
    \frac{1}{T^\alpha} \;.
\]
Assume that $\ts < \tau^\star$ and $\tau^\star + 2 + \ts < \tau_T$ (the other cases can be treated analogously, omitting terms which are not there anymore). 
Then:
\begin{multline*}
    \sum_{t=1}^T f( X_t )
=
    \sum_{\tau=1}^{\ts} \sum_{t = t_{\tau-1} +1}^{t_\tau} f( X_t )
    +
    \sum_{\tau = \ts+1}^{t^\star -1} \sum_{t = t_{\tau-1} +1}^{t_\tau} f( X_t )
\\
    +
    \sum_{\tau = \tau^\star}^{\tau^\star+2} \sum_{t = t_{\tau-1} +1}^{t_\tau} f( X_t )
    +
    \sum_{\tau = \tau^\star+3}^{\tau^\star+2+\ts} \sum_{t = t_{\tau-1} +1}^{t_\tau} f( X_t )
    +
    \sum_{\tau = \tau^\star+3+\ts}^{\tau_T} \sum_{t = t_{\tau-1} +1}^{t_\tau} f( X_t )
\end{multline*}
We analyze these five terms individually.
For the first one, we further split the sum into two terms, depending on whether or not $B_\tau \ge 6$. 
By \Cref{c:claim-1}, we have that
\begin{align*}
    \sum_{\substack{\tau=1\\B_\tau \ge 6}}^{\ts} \sum_{t = t_{\tau-1} +1}^{t_\tau} f( X_t )
&
\le
    \sum_{\substack{\tau=1\\B_\tau \ge 6}}^{\ts} \sum_{t = t_{\tau-1} +1}^{t_\tau} \frac{4c3^\alpha}{(B_{\tau_t}-3)^\alpha}
\le
    \sum_{\substack{\tau=1\\B_\tau \ge 6}}^{\ts} \sum_{t = t_{\tau-1} +1}^{t_\tau} \frac{4c3^\alpha}{(B_{\tau_t}-B_{\tau_t}/2)^\alpha}
=
    \sum_{\substack{\tau=1\\B_\tau \ge 6}}^{\ts} \sum_{t = t_{\tau-1} +1}^{t_\tau} \frac{4c6^\alpha}{B_{\tau_t}^\alpha}
\\
&
=
    \sum_{\substack{\tau=1\\B_\tau \ge 6}}^{\ts} 4c6^\alpha B_{\tau}^{1-\alpha}
\le
    \ts \cdot 4c6^\alpha T^{1-\alpha}
\end{align*}
By \Cref{c:claim-2}, we have that
\[
    \sum_{\substack{\tau=1\\B_\tau \le 5}}^{\ts} \sum_{t = t_{\tau-1} +1}^{t_\tau} f( X_t )
\le
    5 M \sum_{\tau=0}^{\iop} \brb{ \nicefrac34 }^\tau
=
    20 M
\]
Thus, the first term is upper bounded by $\ts \cdot 4c6^\alpha T^{1-\alpha} + 20 M$. 

For the second term, we leverage \Cref{c:claim-2} and the definition of $\ts$ to obtain
\begin{align*}
    \sum_{\tau = \ts+1}^{t^\star -1} \sum_{t = t_{\tau-1} +1}^{t_\tau} f( X_t )
&
\le
    M \sum_{\tau = \ts+1}^{t^\star -1} \sum_{t = t_{\tau-1} +1}^{t_\tau} \brb{ \nicefrac34 }^{\tau-1}
\le
    M \brb{ \nicefrac34 }^{\ts-1} \sum_{\tau = \ts+1}^{t^\star -1} \sum_{t = t_{\tau-1} +1}^{t_\tau} 1
\le
    M \brb{ \nicefrac34 }^{\bfl{\frac{\ts-1}{2}}} \sum_{\tau = \ts+1}^{t^\star -1} \sum_{t = t_{\tau-1} +1}^{t_\tau} 1
\\
&
\le
    T^{1-\alpha}
\end{align*}
For the third term, we further split the sum into two terms, depending on whether or not $B_\tau \ge 6$.
Proceeding exactly as for the first term, we obtain
\[
    \sum_{\tau = \tau^\star}^{\tau^\star+2} \sum_{t = t_{\tau-1} +1}^{t_\tau} f( X_t )
\le
    3 \cdot 4c6^\alpha T^{1-\alpha} + 15 M
\]
For the fourth term, we split again the sum into two terms, depending on whether or not $B_\tau \ge 6$.
If $B_\tau \ge 6$, proceeding exactly as for the corresponding part of the first term, we obtain
\[
    \sum_{\substack{\tau=\tau^\star+3\\B_\tau \ge 6}}^{\tau^\star+2+\ts} \sum_{t = t_{\tau-1} +1}^{t_\tau} f ( X_t )
\le
    \ts \cdot 4c6^\alpha T^{1-\alpha}
\]
Instead, if $B_\tau \le 5$, by \Cref{c:claim-3}, we get
\[
    \sum_{\substack{\tau=\tau^\star+3\\B_\tau \le 5}}^{\tau^\star+2+\ts} \sum_{t = t_{\tau-1} +1}^{t_\tau} f ( X_t )
\le
    5 \sum_{\substack{\tau=\tau^\star+3\\B_\tau \le 5}}^{\tau^\star+2+\ts} \max_{x\in\{ l_\tau, c_\tau, r_\tau \}} f ( x )
\le
    5 \sum_{\substack{\tau=\tau^\star+3\\B_\tau \le 5}}^{\tau^\star+2+\ts} \max_{x \in \lcb{ x^{m_\tau}_1, \dots, x^{m_\tau}_8 }} f ( x )
\le
    10 M \sum_{\tau= 0}^{\iop} \brb{\nicefrac47}^\tau
\le
    \frac{70}3 M \;.
\]
For the last term, by \Cref{c:claim-3}, we get
\begin{align*}
    \sum_{\tau=\tau^\star+3+\ts}^{\tau_T} \sum_{t = t_{\tau-1} +1}^{t_\tau} f( X_t )
&
\le
    \sum_{\tau=\tau^\star+3+\ts}^{\tau_T} \sum_{t = t_{\tau-1} +1}^{t_\tau} \max_{x \in \lcb{ x^{m_\tau}_1, \dots, x^{m_\tau}_8 }} f(x)
\le
    \sum_{\tau=\tau^\star+3+\ts}^{\tau_T} \sum_{t = t_{\tau-1} +1}^{t_\tau} M \brb{ \nicefrac47 }^{\lfl{ \frac{\tau-(\tau^\star+3) -1}{2} }}
\\
&
\le
    M \brb{ \nicefrac34 }^{\lfl{ \frac{\ts-1}{2} }} \sum_{\tau=\tau^\star+3+\ts}^{\tau_T} \sum_{t = t_{\tau-1} +1}^{t_\tau} 1
\le
    T^{1-\alpha} \;.
\end{align*}
Putting everything together, we conclude that
\begin{align*}
    R_T
&
\le
    \brb{
    \ts \cdot 4c6^\alpha T^{1-\alpha}
    +
    20 M
    }
    +
    T^{1-\alpha}
    +
    \brb{
    \ts \cdot 4c6^\alpha T^{1-\alpha}
    +
    15 M
    }
    +
    \frac{70}3 M
    +
    T^{1-\alpha}
\\
&
\le
    \brb{ \bfl{ 4 + 2 \log_{\nicefrac43}(MT^\alpha) } \cdot 8c6^\alpha +2 }T^{1-\alpha}
    +
    60 M \;.
\end{align*}
\end{proof}

\bibliographystyle{plainnat}
\bibliography{biblio}

\section*{Acknowledgments}
Tommaso Cesari gratefully acknowledges the support of IBM.

\appendix

\end{document}